\documentclass{article}

\usepackage{amsthm}

\usepackage{arxiv}

\usepackage{algorithm}
\usepackage{algorithmic}

\usepackage{amsmath}
\usepackage{txfonts}

\usepackage[utf8]{inputenc} 
\usepackage[T1]{fontenc}    
\usepackage{hyperref}       
\usepackage{url}            
\usepackage{booktabs}       
\usepackage{amsfonts}       
\usepackage{nicefrac}       
\usepackage{microtype}      
\usepackage{cleveref}       
\usepackage{lipsum}         
\usepackage{graphicx}
\usepackage{natbib}
\usepackage{doi}

\usepackage{amsfonts}
\usepackage{amssymb}

\usepackage{tcolorbox}
\tcbuselibrary{breakable,theorems,skins}

\usepackage[english]{babel}

\newcounter{theo} 
\newcounter{lemm} 
\newtcbtheorem
{proposition}
{Proposition}
{enhanced,before title={\stepcounter{theo}},colback=blue!10,colframe=blue!35!black,fonttitle=\bfseries,%
attach boxed title to top left={xshift=5mm,yshift*=-\tcboxedtitleheight/2},%
boxed title style={colback=blue!35!black}}
{th}%

\newtcbtheorem
{lemma}
{Lemma}
{enhanced,before title={\stepcounter{lemm}},colback=green!10,colframe=green,fonttitle=\bfseries,colbacktitle=green!10,coltitle=black,%
attach boxed title to top left={xshift=5mm,yshift*=-\tcboxedtitleheight/2},%
boxed title style={boxrule=0.6pt}}%
{co}%

\title{Interpretation of the Transformer \\ and Improvement of the Extractor}

\date{}

\newif\ifuniqueAffiliation
\uniqueAffiliationtrue

\ifuniqueAffiliation 
\author{ {\hspace{1mm}Zhe~Chen} \\
	School of Computer Science and Engineering\\
	Northeastern University \\
          Shenyang, Liaoning, China\\
	\texttt{ml\_iot@163.com; chenzhe@mail.neu.edu.cn} \\
}
\else
\usepackage{authblk}

\newbox{\orcid}\sbox{\orcid}{\includegraphics[scale=0.06]{orcid.pdf}} 
\fi

\renewcommand{\headeright}{Zhe Chen}
\renewcommand{\undertitle}{}
\renewcommand{\shorttitle}{Interpretation of the Transformer and Improvement of the Extractor}

\begin{document}
\maketitle

\begin{abstract}

It has been over six years since the Transformer architecture was put forward. Surprisingly, the vanilla Transformer architecture is still widely used today. One reason is that the lack of deep understanding and comprehensive interpretation of the Transformer architecture makes it more challenging to improve the Transformer architecture. In this paper, we first interpret the Transformer architecture comprehensively in plain words based on our understanding and experiences. The interpretations are further proved and verified. These interpretations also cover the Extractor, a family of drop-in replacements for the multi-head self-attention in the Transformer architecture. Then, we propose an improvement on a type of the Extractor that outperforms the self-attention, without introducing additional trainable parameters. Experimental results demonstrate that the improved Extractor performs even better, showing a way to improve the Transformer architecture.

\end{abstract}


\section{Introduction}

The Transformer architecture, since it launched in 2017~\citep{van17}, has been widely adopted in natural language processing, computer vision, and many other areas. Nowadays, various large language models (LLMs) are trained using the Transformer architecture.

However, the lack of deep understanding and comprehensive interpretation hinders further improvements on the Transformer architecture. 
Moreover, the lack of interpretability also limits applications in real-world~\citep{YangHZZDC23}.
As a result, the vanilla Transformer architecture is still widely used today, over six years after its debut, even in the age of AI (artificial intelligence) rush. One reason is that interpreting the Transformer is a low-level, complex, demanding, time-consuming, and non-profitable task that fewer people would like to take.

In~\cite{chen2023attention}, a family of Extractors is proposed to replace the multi-head self-attention in the Transformer in a drop-in fashion. Specifically, a type of the Extractor called the higher-performance Extractor (HE) is capable of outperforming the multi-head self-attention with fewer arithmetic operations and the same number of trainable parameters. And a more powerful type of the Extractor called the super high-performance Extractor (SHE) achieves a much better performance than the multi-head self-attention.

In this paper, we first interpret the Transformer architecture, as well as what the self-attention and the Extractor actually do, based on our understanding and experiences. These interpretations are further proved and verified. Then, we propose an improvement on the SHE. Experimental results demonstrate that with the improvement the SHE can achieve a better performance. To our best knowledge, this is the first time that the Transformer architecture is comprehensively interpreted in plain words.

Our contributions are summarized as follows.
\begin{itemize}
  \item We comprehensively interpret the Transformer architecture (as well as what the self-attention and the Extractor actually do) in plain words.
  \item We prove and further verify the interpretations.
  \item We propose an improvement on the SHE, which is a high-performance replacement of the multi-head self-attention in the Transformer.
  \item We evaluate the performance of the Transformers equipped with the improved SHE in text generation.
\end{itemize}

\section{Interpretation of the Transformer}

\begin{figure}[!t]
  \centering
  \includegraphics[width=0.7\columnwidth]{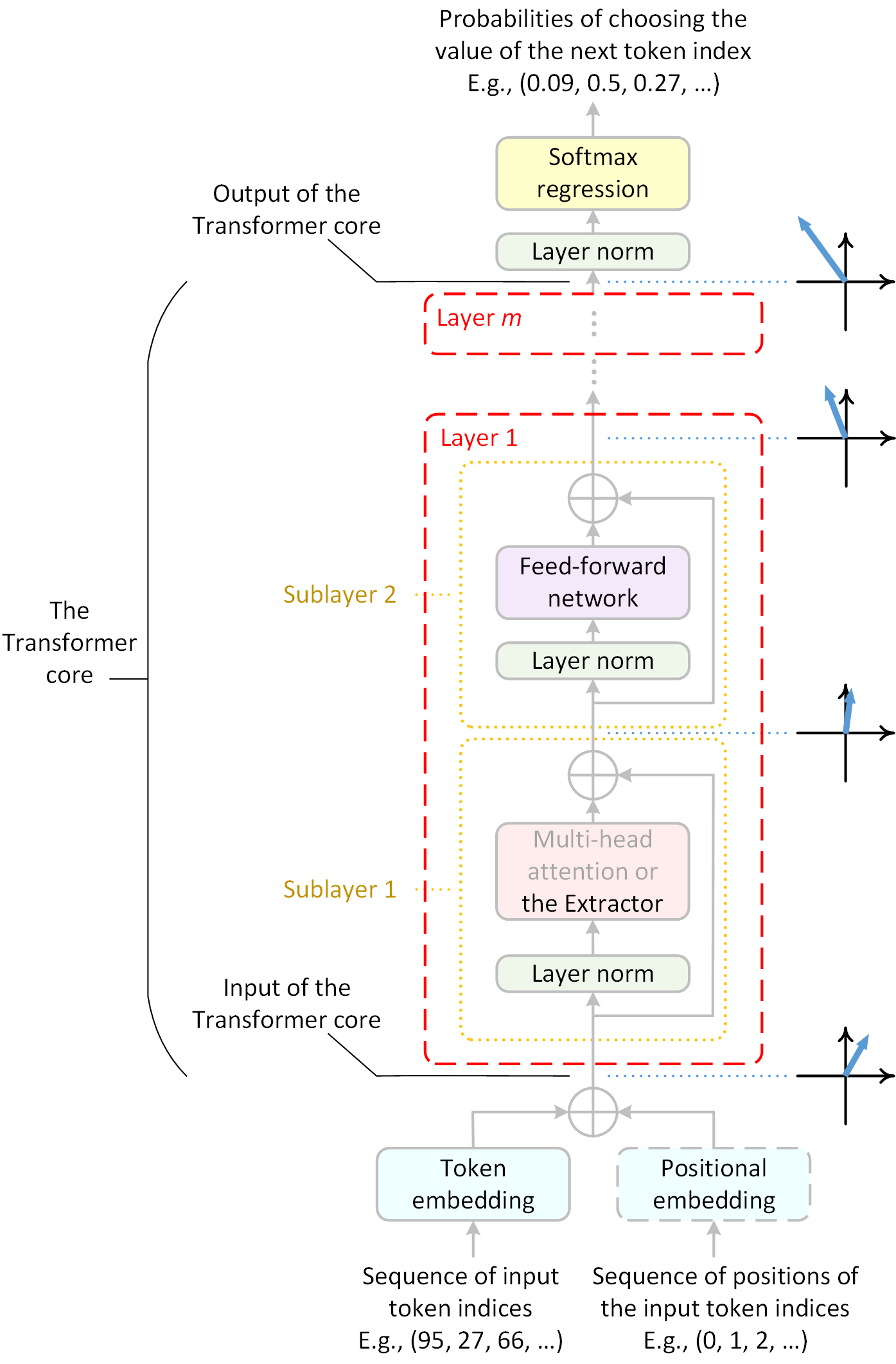} 
  \caption{Illustration of the Transformer.}
  \label{fig_trans}
\end{figure}

In this paper, we use text generation as an example task for the Transformer. As discussed in~\cite{chen2023attention}, the Transformer architecture is employed to build models to predict or infer the probabilities of the next token given a variable-length sequence of tokens. 

Alternatively, we can view those variable-length sequences as fixed-length sequences if we introduce a padding token whose embedding is always a zero vector. That is to say, we can virtually add padding tokens at the beginning of a variable-length sequence to make its length equal to a given number $l$, where $l$ is the maximum sequence length that a model supports and the length of the context window in text generation. Since the embeddings of padding tokens are zero vectors, padding tokens have no effect on the output of the Transformer. In this way, we convert the “variable-length” problem into a “fixed-length” problem, meaning that we can regard all the input sequences to the Transformer as fixed-length sequences. 

In text generation, the outputs of the Transformer are probabilities for choosing the next token. Since the last component of the Transformer is virtually a softmax regression~\citep{chen2023attention}, these probabilities come from the input vector of the softmax regression, whereas the input vector of the softmax regression comes from the output of the last Transformer layer. The number of elements in these vectors is $d$, where $d$ is the internal “dimension” of the Transformer. 
To facilitate discussions, we introduce the term ``Transformer core'' to refer to the stack of $m$ Transformer layers, as illustrated in Fig.~\ref{fig_trans}. In the following discussions, we will focus on the Transformer core, for the interpretations of both softmax regression and embedding are pretty clear.

\begin{lemma}{}{}
The Transformer core implements a matrix function $f:\ \mathbb{R}^{t \times d}\rightarrow\mathbb{R}^{t \times d}$, where $t$ is the length of the input sequence.
\end{lemma}
\begin{proof}
The Transformer inputs a sequence of indices of $t$ tokens. The embedding layer maps the input sequence into a $t \times d$ matrix, which is the input of the Transformer core. Thus, the domain of the function is $\mathbb{R}^{t \times d}$. On the other hand, the Transformer core actually outputs a $t \times d$ matrix, although in the phase of inference only the last row of this matrix is used by the succedent softmax regression. Hence, the codomain of the function is also $\mathbb{R}^{t \times d}$.
\end{proof}

\begin{lemma}{}{}
 Each layer (e.g., the $k$-th layer) in the Transformer core implements a matrix function $f^{[k]}:\ \mathbb{R}^{t \times d}\rightarrow\mathbb{R}^{t \times d}$, where $k=1,2,\cdots,m$ and $m$ is the number of layers in the Transformer core.
\end{lemma}
\begin{proof}
There are $m$ serially-connected layers in the Transformer core. These layers are identical in terms of architecture, although the values of their parameters may be different. They all input $t \times d$ matrices and output $t \times d$ matrices, meaning that both their domains and their codomains are $\mathbb{R}^{t \times d}$. As the values of their parameters may be different, they may implement different functions.
\end{proof}

\begin{proposition}{}{}
The Transformer core implements a composite function $f=f^{[m]} \circ f^{[m-1]} \circ \cdots \circ f^{[1]}$.
\end{proposition}
\begin{proof}
All the layers in the Transformer core are serially connected, meaning that the output of the $k$-th layer is the input of the $(k+1)$-th layer, where $k=1,2,\cdots,m-1$. And all the domains and codomains of $f^{[1]}, f^{[2]}, \cdots, f^{[m]}$ are the same.
\end{proof}

As shown in Fig.~\ref{fig_trans}, a standard Transformer layer consists of two serially-connected sublayers: the self-attention sublayer or the Extractor sublayer as the first sublayer (sublayer 1) and the feed-forward network (FFN) sublayer as the second sublayer (sublayer 2). Please note that the order of the two sublayers may alter and a Transformer layer may consist of only one sublayer. Besides, it is also possible for the two sublayers to be connected in parallel, as proposed in~\cite{chowdhery23} and~\cite{zhong2022neural}. In order to keep it consistent with the vanilla Transformer, in this paper we refer the self-attention sublayer or the Extractor sublayer to sublayer 1 and refer the FFN sublayer to sublayer 2.

\begin{proposition}{}{}
Sublayer 1 performs dimensionality reduction.
\end{proposition}
\begin{proof}
The input of sublayer 1 is the input of a Transformer layer, which is a $t \times d$ matrix. The output of sublayer 1 is also a $t \times d$ matrix. The $i$-th row of the output matrix are computed using the $1$-st, $2$-nd , $\cdots$, and $i$-th rows of the input matrix, where $i=1,2,\cdots,t$. Those rows of the input matrix can be regarded as an $id$ -vector. And the $i$-th row of the output matrix can be regarded as a $d$-vector. In this sense, sublayer 1 reduces $id$-vectors to $d$-vectors. This is what the self-attention sublayer does.
As we discussed earlier, with padding tokens we can virtually extend an input $i \times d$ matrix to an $l \times d$ matrix, where $i=1,2,\cdots,t$ and $t \le l$. In this sense, sublayer 1 reduces $ld$-vectors to $d$-vectors. This is what the Extractor does. In either case, sublayer 1 reduces dimensionalities.
\end{proof}

Alternately, from the aspect of encoding, the aforementioned dimensionality reduction can be viewed as an encoding process, for it converts $id$-vectors or $ld$-vectors into $d$-vectors. In this sense, an output $d$-vector is a ``code'' that ideally corresponds to an input $id$-vector or $ld$-vector. Both the self-attention and the Extractor can be regarded as encoders. 

\begin{proposition}{}{}
Sublayer 2 performs transformation.
\end{proposition}
\begin{proof}
The input and output of sublayer 2 are both $t \times d$ matrices. The $i$-th row of its output matrix is computed only using the $i$-th row of its input matrix, where $i=1,2,\cdots,t$. And each row of the input and output matrices is a $d$-vector. In this sense, sublayer 2 transforms $d$-vectors to $d$-vectors.
\end{proof}

Alternately, from the aspect of mapping, sublayer 2 maps each row of its input matrix, or ``code'', into a corresponding output row vector.

As shown in Fig.~\ref{fig_trans}, there is a residual connection in each sublayer. In practice, both pre-layer normalizations and dropouts are commonly applied within the residual connections. Since pre-layer normalization is a feature scaling technique in nature and dropout is a regularization technique and they are not quite relevant to the interpretation of the Transformer architecture, they will not be considered in this section, in order to keep the discussions simpler.

\begin{proposition}{}{}
The residual connections in sublayers contribute to expediting the training of Transformer models.
\end{proposition}
\begin{proof}
In the phase of training, we usually use random numbers whose absolute values are close to zeros to initialize the weights of Transformer models to avoid divergence. And for the same reason, the initial learning rate used in training is small, too. Without residual connections, this may cause the absolute values of the elements in the output matrices of the sublayers to be small in the early phase of training, resulting in smaller update steps for weights in each training iteration, especially when the number of layers is large. Whereas with residual connections, the absolute values of the elements in the output matrices of the sublayers are larger in general in the early phase of training, resulting in larger update steps for weights, which expedites the training of Transformer models.
\end{proof}

\begin{proposition}{}{}
With the residual connection, sublayer 1 adjusts the row vectors of its input matrix based on the same and prior row vectors of its input matrix.
\end{proposition}
\begin{proof}
The $i$-th row vector in the output matrix of the self-attention or the Extractor is computed using the $1$-st, $2$-nd, $\cdots$, and $i$-th row vectors of its input matrix, where $i=1, 2, \cdots, t$. With the residual connection, the $i$-th row vector in the output matrix of sublayer 1 is the resultant vector of the $i$-th row vector in the input matrix of sublayer 1 and the $i$-th row vector in the output matrix of the self-attention or the Extractor, meaning that the $i$-th row vector in the input matrix of sublayer 1 is adjusted by the $i$-th row vector in the output matrix of the self-attention or the Extractor. 
\end{proof}

\begin{proposition}{}{}
With the residual connection, sublayer 2 adjusts the row vectors of its input matrix based on the same row vectors of its input matrix.
\end{proposition}
\begin{proof}
The $i$-th row vector in the output matrix of the FFN is computed just using the $i$-th row vector of its input matrix, where $i=1, 2, \cdots, t$. With the residual connection, the $i$-th row vector in the output matrix of sublayer 2 is the resultant vector of the $i$-th row vector in the input matrix of sublayer 2 and the $i$-th row vector in the output matrix of the FFN, meaning that the $i$-th row vector in the input matrix of sublayer 2 is adjusted by the $i$-th row vector in the output matrix of the FFN. 
\end{proof}

\begin{proposition}{}{}
The Transformer core maps its input matrix into output matrix by driving the row vectors in its input matrix towards the row vectors in its output matrix layer by layer.
\end{proposition}
\begin{proof}
According to Lemma 1, the transformer core maps a $t \times d$ matrix into a $t \times d$ matrix. Moreover, according to Proposition 5 and Proposition 6, both sublayers in a Transformer layer adjust the row vectors in its input matrix. Therefore, each layer in the Transformer core drives the row vectors in its input matrix towards the row vectors in its output matrix. Hence, the Transformer core drives the row vectors in its input matrix towards the row vectors in its output matrix, layer by layer.
\end{proof}

In summary, the Transformer first converts a sequence of token indices (as well as their positions) into a matrix via embedding. Then, the Transformer core maps this matrix into another matrix  by driving the row vectors in the input matrix towards the row vectors in its output matrix layer by layer, as  illustrated in Fig.~\ref{fig_trans}. The row vectors in the output matrix represent the predictions that the Transformer makes. Finally, the softmax regression in the Transformer maps each row vector in the output matrix into probabilities for choosing the next token in a vocabulary. Please note that in the phase of inference only the last row vector in the output matrix is used whereas in the phase of training all the row vectors may be used.

The self-attention sublayer in the Transformer provides a way to reduce multiple $d$-vectors (i.e., multiple row vectors in the input $t \times d$ matrix) to one $d$-vector. It generates dynamic weights (by the self-attention mechanism) to weight the multiple $d$-vectors, resulting in much diversified outputs or ``codes'', which is its advantage. The Extractor sublayer, on the other hand, uses static weights to weight the multiple $d$-vectors (i.e., the row vectors in the virtual input $l \times d$ matrix) and employs dynamic element-wise multiplications (as what the ``adjustment'' part does) to diversify its outputs or ``codes''. Thus, the Extractor is able to reduce the computational complexity of sublayer 1 while maintaining the performance of the Transformer. Moreover, another advantage of the Extractor sublayer is that it does not require positional embeddings since it uses position-relevant static weights to weight the multiple $d$-vectors or the multiple $d$-vectors are computed using position-relevant weights. As an example, a type of the Extractor called the HE (as proposed in~\cite{chen2023attention}) is able to outperform the multi-head self-attention with fewer arithmetic operations and the same number of trainable parameters.

\section{Improvement of the Extractor}

\begin{figure}[!t]
  \centering
  \includegraphics[width=0.8\columnwidth]{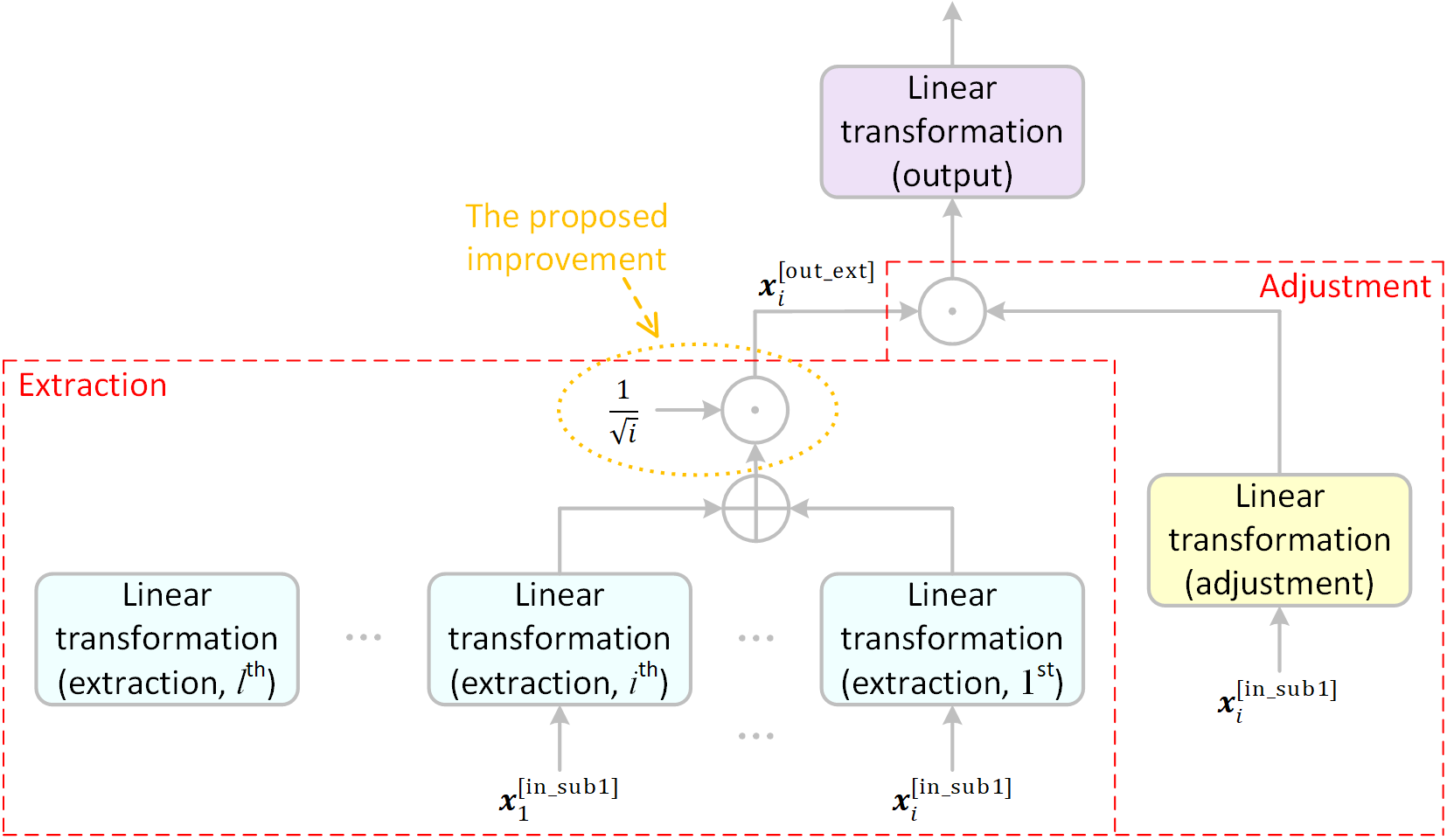} 
  \caption{The proposed iSHE.}
  \label{fig_ishe}
\end{figure}

In~\cite{chen2023attention}, a type of the Extractor called the SHE is proposed to replace the multi-head self-attention in a drop-in fashion. Although it outperforms the self-attention, we find that its performance can be further improved. In this section, we propose an improvement to the SHE.

The Fig. 3 in~\cite{chen2023attention} illustrates the SHE. An imperfection of the SHE is that the SHE does not scale the output of its ``extraction'' part (i.e., the resultant vector of a number of vectors) in accordance with the length of the sequence (i.e., the number of vectors). This results in that the variances of the elements in the output matrix of the ``extraction'' part vary with the length of the sequence, causing the variances of the elements in the output matrix of sublayer 1 to vary with the length of the sequence, which deteriorates the performance of the Transformer.

To address this issue, we propose to standardize the variances of the elements in the output matrix of the ``extraction'' part by multiplying $\frac{1}{\sqrt{i}}$, as shown in Fig.~\ref{fig_ishe} and Eq.~(\ref{eq_out_ext}). 
\begin{equation} 
  \label{eq_out_ext}
  \boldsymbol{x}_i^{[\mathrm{out\_ext}]}= \frac{1}{\sqrt{i}} \sum_{j=1}^i \boldsymbol{x}_j^{[\mathrm{in\_sub1}]} \boldsymbol{W}_{i-j+1}^{[\mathrm{ext}]}
\end{equation} 
where $\boldsymbol{x}_i^{\left[\mathrm{out\_ext}\right]}$ is the $i$-th row vector in the output matrix $\boldsymbol{X}^{\left[\mathrm{out\_ext}\right]}$, $\boldsymbol{X}^{[\mathrm{out\_ext}]}\in\mathbb{R}^{t \times d}$, $\boldsymbol{x}_i^{\left[\mathrm{out\_ext}\right]}\in\mathbb{R}^{1\times d}$, $\boldsymbol{x}_j^{\left[\mathrm{in\_sub1}\right]}$ is the $j$-th row vector in the input matrix $\boldsymbol{X}^{\left[\mathrm{in\_sub1}\right]}$, $\boldsymbol{X}^{[\mathrm{in\_sub1}]}\in\mathbb{R}^{t \times d}$, $\boldsymbol{x}_j^{\left[\mathrm{in\_sub1}\right]}\in\mathbb{R}^{1\times d}$,  $\boldsymbol{W}_1^{[\mathrm{ext}]}, \boldsymbol{W}_2^{[\mathrm{ext}]},\cdots,\boldsymbol{W}_l^{[\mathrm{ext}]}$ are weight matrices, $\boldsymbol{W}_1^{[\mathrm{ext}]}, \boldsymbol{W}_2^{[\mathrm{ext}]},\cdots,\boldsymbol{W}_l^{[\mathrm{ext}]}\in\mathbb{R}^{d\times d}$, and $i=1,2,\cdots,t$. 

We call this improved version of SHE iSHE in this paper.

\section{Experiments}

In this section, we first experimentally validate our interpretation of the Transformer. Then, we evaluate the performance of the proposed iSHE.

\subsection{Validation of the Interpretation of the Transformer}

Since Proposition 7 interprets the core idea of the Transformer, we focus on the validation of Proposition 7 in the following experiment. In order to faithfully plot $d$-vectors in the Euclidean plane, we let $d=2$. However, $d$ is one of the major hyperparameters that decide the capacity of a Transformer model. Thus, it is reasonable to employ a small vocabulary when $d$ is small. On the other hand, in order to train a Transformer model even with few layers and a very small vocabulary, lots of training examples are required, meaning that the training dataset should be large enough. So we generate a dataset with over 220,000 training examples and only three tokens in the vocabulary. To be exact, the three tokens are ``0'', ``1'', and ``;'', respectively. This dataset simply contains $2^{14}$ binary numbers (ranging from $0$ to $2^{14} - 1$) separated by semicolons.

With this dataset, two Transformer models are trained. One model employs the 1-head self-attention sublayer, whereas the other employs the SHE sublayer. The hyperparameters and settings for training these Transformer models are listed in Table~\ref{table_param_val}. All the biases in these models are disabled, since in general they do not contribute to improving the performance of the Transformer. The parameter-free version of pre-layer normalization is used. And dropout is not used. Fig.~\ref{fig_vectors_att1} and Fig.~\ref{fig_vectors_she} show the output row vectors of all the sublayers in the Transformer core when the Transformer core inputs an row vector.
It can be seen that as the starting point the input row vector is driven to the end point (i.e., the output of sublayer 2 in layer 4) step by step by every sublayer in the Transformer core, no matter what type of sublayer 1 is employed. This experiment validates Proposition 7.

\begin{table*}[t]
\centering
\begin{tabular}{l|l}
    \toprule
    Hyperparameter or setting & Value \\
   \midrule
    Size of the vocabulary & 3 \\
    Length of the context window ($l$) & 64 \\
    ``Dimension'' of the model ($d$) & 2 \\
    Number of the nodes in the hidden layer of the FFN & 8 \\
    Number of layers ($m$) & 4 \\
    Batch size & 64 \\
    Number of batches for training & 3500 (0.977 epochs) \\
    Optimizer & AdamW ($\gamma=0.001$, $\beta_1=0.9$, $\beta_2=0.999$) \\
    Learning rate & 0.001 \\
    Standard deviation for initializing weights & 0.01 \\
    \bottomrule
\end{tabular}
\caption{Hyperparameters and settings for the validation.}
\label{table_param_val}
\end{table*}

\begin{figure}[!t]
  \centering
  \includegraphics[width=0.75\columnwidth]{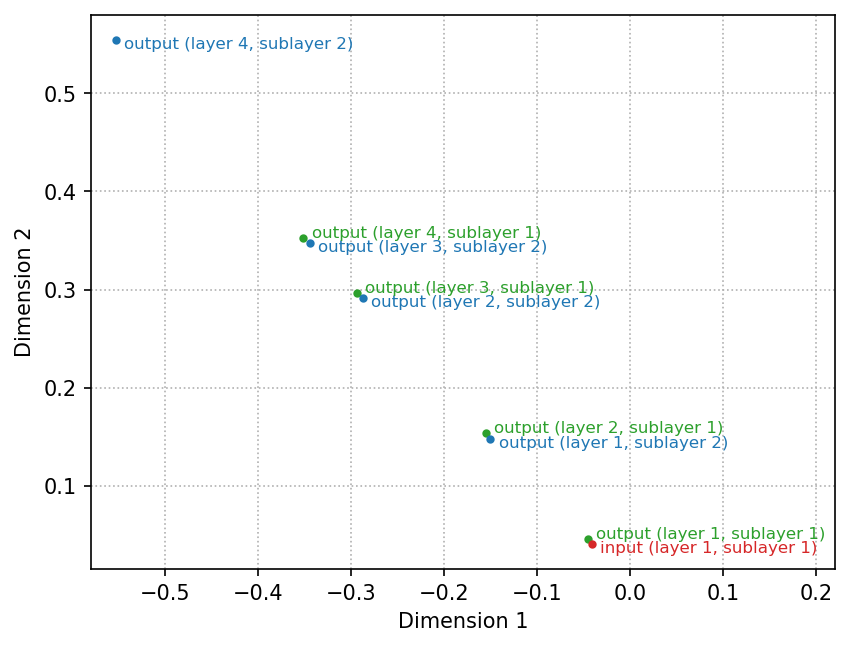} 
  \caption{The input and output row vectors of the sublayers in the Transformer core with the 1-head self-attention sublayer.}
  \label{fig_vectors_att1}
\end{figure}

\begin{figure}[!t]
  \centering
  \includegraphics[width=0.75\columnwidth]{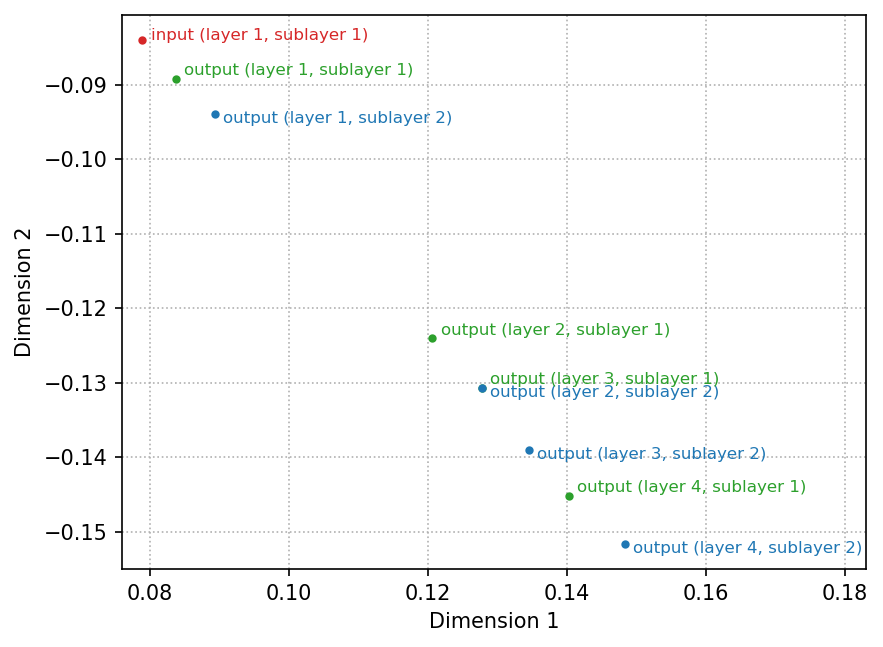} 
  \caption{The input and output row vectors of the sublayers in the Transformer core with the SHE sublayer.}
  \label{fig_vectors_she}
\end{figure}

\subsection{Evaluation of the Performance of the proposed iSHE}

In order to evaluate the performance of the proposed iSHE, nine Transformer models with either the self-attention sublayer (with 1, 2, 4, 8, 16, 32, and 64 heads, respectively) or the Extractor sublayer (the SHE or the iSHE) are trained using exactly the same dataset that is used in~\cite{chen2023attention} for text generation. This dataset is composed of the top 100 books on English children's literature available at gutenberg.org, a library of free ebooks. The raw text of the books is tokenized using the Hugging Face BPE (byte-pair encoding) tokenizer with a vocabulary size of 5000, resulting in a total of 8.4M tokens.

The models are implemented and trained using the PyTorch 2.1 framework on an NVIDIA GeForce RTX 4050 GPU (graphics processing unit). The hyperparameters and settings for this experiment are listed in Table~\ref{table_param_eva}. As what we did in the previous experiment, all the biases in these models are disabled and dropout is not used. And the parameter-free version of pre-layer normalization is used. The weights are initialized randomly following a normal distribution with a mean of zero and a standard deviation of 0.01. All the models are initialized and trained with the same random seed.
\begin{table*}[t]
\centering
\begin{tabular}{l|l}
    \toprule
    Hyperparameter or setting & Value \\
   \midrule
    Size of the vocabulary & 5000 \\
    Length of the context window ($l$) & 128 \\
    ``Dimension'' of the model ($d$) & 128 \\
    Number of the nodes in the hidden layer of the FFN & 512 \\
    Number of layers ($m$) & 12 \\
    Batch size & 64 \\
    Number of batches for training & 30000 (0.228 epochs) \\
    Optimizer & AdamW ($\gamma=0.001$, $\beta_1=0.9$, $\beta_2=0.999$) \\
    Learning rate & 0.001 \\
    Standard deviation for initializing weights & 0.01 \\
    \bottomrule
\end{tabular}
\caption{Hyperparameters and settings for the evaluation.}
\label{table_param_eva}
\end{table*}

\begin{figure}[!t]
  \centering
  \includegraphics[width=0.75\columnwidth]{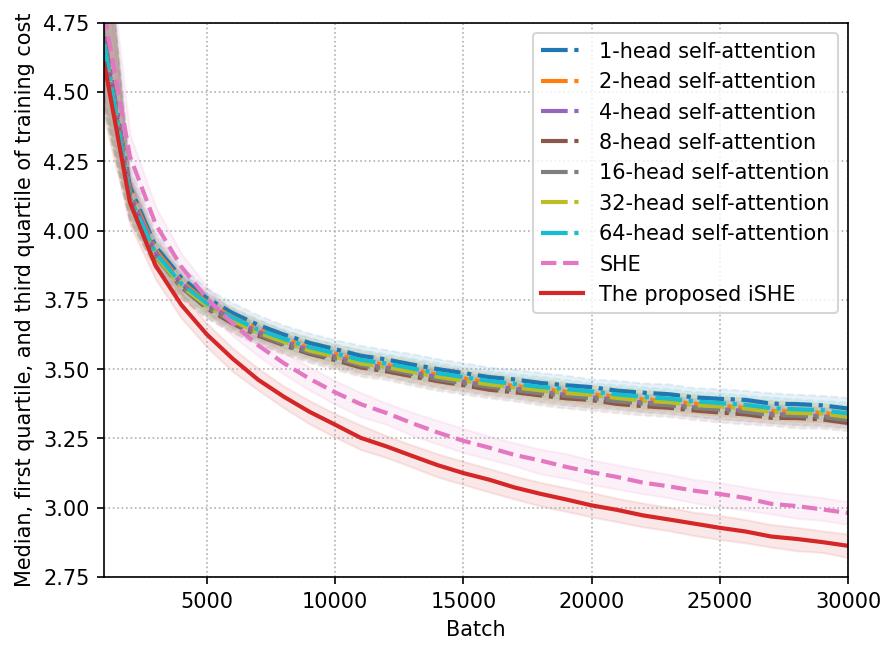} 
  \caption{Median, first quartile, and third quartile of the training costs of the models with different types of sublayer 1.}
  \label{fig_eval}
\end{figure}
We use training cost (i.e., the average training loss over a batch) as the evaluation matric since training cost equals perplexity in this task. Perplexity measures how well a probability model predicts. The lower the perplexity, the better the model predicts.
Fig.~\ref{fig_eval} shows the median, the first quartile, and the third quartile of the training costs for every non-overlapping 1000 batches. This figure evidently indicates that the model with the proposed iSHE sublayer outperforms all the other models. 
Please note that such a performance gain does not cost a single extra trainable parameter. Furthermore, both the iSHE and the SHE do not require the positional embedding that is required by the self-attention, saving $l \cdot d$ trainable parameters compared with the self-attention in this case.

\section{Related Work}

There are quite a few works related to interpreting the Transformer. Most of them focus on interpreting the attention matrix, specific network components, model parameters, or hidden representations.

The most related work we are aware of is~\cite{molina2023traveling}. This work connects the ideas proposed in previous works and interprets the Transformer based on the geometric interpretation of layer normalization. However, layer normalization is not an indispensable ingredient to the Transformer, as the Transformer can work without layer normalizations. In contrast, our interpretations do not rely on layer normalization.

\section{Conclusion}

In this paper, the Transformer architecture is comprehensively interpreted in plain words. Although we focus on the decoder of the Transformer, the interpretations can also be applied to the encoder. 

Moreover, a type of the Extractor, namely the SHE, is improved without introducing additional trainable parameters. The improved SHE, or the iSHE, achieves a better performance by simply introducing a multiplier factor to the output of the ``extraction'' part. Therefore, we strongly recommend to replace the SHE with the proposed iSHE.

\bibliographystyle{unsrtnat}
\bibliography{references}

\begin{thebibliography}{6}
\providecommand{\natexlab}[1]{#1}
\providecommand{\url}[1]{\texttt{#1}}
\expandafter\ifx\csname urlstyle\endcsname\relax
  \providecommand{\doi}[1]{doi: #1}\else
  \providecommand{\doi}{doi: \begingroup \urlstyle{rm}\Url}\fi

\bibitem[Vaswani et~al.(2017)Vaswani, Shazeer, Parmar, Uszkoreit, Jones, Gomez,
  Kaiser, and Polosukhin]{van17}
Ashish Vaswani, Noam Shazeer, Niki Parmar, Jakob Uszkoreit, Llion Jones,
  Aidan~N. Gomez, \L{}ukasz Kaiser, and Illia Polosukhin.
\newblock Attention is all you need.
\newblock In \emph{Proceedings of the 31st International Conference on Neural
  Information Processing Systems}, NIPS'17, page 6000–6010, Red Hook, NY,
  USA, 2017. Curran Associates Inc.

\bibitem[Yang et~al.(2023)Yang, Huang, Zou, Zhang, Dai, and Chen]{YangHZZDC23}
Sen Yang, Shujian Huang, Wei Zou, Jianbing Zhang, Xinyu Dai, and Jiajun Chen.
\newblock Local interpretation of transformer based on linear decomposition.
\newblock In Anna Rogers, Jordan~L. Boyd{-}Graber, and Naoaki Okazaki, editors,
  \emph{Proceedings of the 61st Annual Meeting of the Association for
  Computational Linguistics (Volume 1: Long Papers), {ACL} 2023, Toronto,
  Canada, July 9-14, 2023}, pages 10270--10287. Association for Computational
  Linguistics, 2023.
\newblock URL \url{https://doi.org/10.18653/v1/2023.acl-long.572}.

\bibitem[Chen(2023)]{chen2023attention}
Zhe Chen.
\newblock Attention is not all you need anymore, 2023.
\newblock URL \url{https://doi.org/10.48550/arXiv.2308.07661}.

\bibitem[Chowdhery et~al.(2023)Chowdhery, Narang, Devlin, Bosma, Mishra,
  Roberts, Barham, Chung, Sutton, Gehrmann, Schuh, Shi, Tsvyashchenko, Maynez,
  Rao, Barnes, Tay, Shazeer, Prabhakaran, Reif, Du, Hutchinson, Pope, Bradbury,
  Austin, Isard, Gur{-}Ari, Yin, Duke, Levskaya, Ghemawat, Dev, Michalewski,
  Garcia, Misra, Robinson, Fedus, Zhou, Ippolito, Luan, Lim, Zoph, Spiridonov,
  Sepassi, Dohan, Agrawal, Omernick, Dai, Pillai, Pellat, Lewkowycz, Moreira,
  Child, Polozov, Lee, Zhou, Wang, Saeta, Diaz, Firat, Catasta, Wei,
  Meier{-}Hellstern, Eck, Dean, Petrov, and Fiedel]{chowdhery23}
Aakanksha Chowdhery, Sharan Narang, Jacob Devlin, Maarten Bosma, Gaurav Mishra,
  Adam Roberts, Paul Barham, Hyung~Won Chung, Charles Sutton, Sebastian
  Gehrmann, Parker Schuh, Kensen Shi, Sasha Tsvyashchenko, Joshua Maynez,
  Abhishek Rao, Parker Barnes, Yi~Tay, Noam Shazeer, Vinodkumar Prabhakaran,
  Emily Reif, Nan Du, Ben Hutchinson, Reiner Pope, James Bradbury, Jacob
  Austin, Michael Isard, Guy Gur{-}Ari, Pengcheng Yin, Toju Duke, Anselm
  Levskaya, Sanjay Ghemawat, Sunipa Dev, Henryk Michalewski, Xavier Garcia,
  Vedant Misra, Kevin Robinson, Liam Fedus, Denny Zhou, Daphne Ippolito, David
  Luan, Hyeontaek Lim, Barret Zoph, Alexander Spiridonov, Ryan Sepassi, David
  Dohan, Shivani Agrawal, Mark Omernick, Andrew~M. Dai,
  Thanumalayan~Sankaranarayana Pillai, Marie Pellat, Aitor Lewkowycz, Erica
  Moreira, Rewon Child, Oleksandr Polozov, Katherine Lee, Zongwei Zhou, Xuezhi
  Wang, Brennan Saeta, Mark Diaz, Orhan Firat, Michele Catasta, Jason Wei,
  Kathy Meier{-}Hellstern, Douglas Eck, Jeff Dean, Slav Petrov, and Noah
  Fiedel.
\newblock Palm: Scaling language modeling with pathways.
\newblock \emph{Journal of Machine Learning Research}, 24:\penalty0
  240:1--240:113, 2023.
\newblock URL \url{http://jmlr.org/papers/v24/22-1144.html}.

\bibitem[Zhong et~al.(2022)Zhong, Zhang, Chakraborty, and Dey]{zhong2022neural}
Yaofeng~Desmond Zhong, Tongtao Zhang, Amit Chakraborty, and Biswadip Dey.
\newblock A neural {ODE} interpretation of transformer layers.
\newblock In \emph{Proceedings of the DLDE Workshop in the 36th Conference on
  Neural Information Processing Systems}, 2022.

\bibitem[Molina(2023)]{molina2023traveling}
Raul Molina.
\newblock Traveling words: {A} geometric interpretation of transformers, 2023.
\newblock URL \url{https://doi.org/10.48550/arXiv.2309.07315}.

\end{thebibliography}

\end{document}